\documentclass[twoside]{article}

\usepackage[accepted]{aistats2024}

\usepackage[style=authoryear,giveninits=true,uniquename=false,uniquelist=false,maxbibnames=99]{biblatex}
\usepackage[colorlinks=true,citecolor=blue,urlcolor=blue]{hyperref}
\usepackage{amsfonts}
\usepackage{amsthm}
\usepackage{mathtools}
\usepackage{tikz}
\usepackage{caption}
\usepackage{subcaption}
\usepackage{appendix}
\usepackage[ruled]{algorithm2e}

\DeclareSymbolFont{sfletters}{OML}{cmbrm}{m}{it}
\DeclareMathSymbol{\svarepsilon}{\mathord}{sfletters}{"22}

\makeatletter
\newcommand{\removelatexerror}{\let\@latex@error\@gobble}
\makeatother

\newtheorem{theorem}{Theorem}
\newtheorem{proposition}{Proposition}
\newtheorem{lemma}{Lemma}

\renewbibmacro{in:}{}
\DeclareNameAlias{sortname}{family-given}
\DeclareNameAlias{default}{family-given}

\newcommand\rurl[1]{\href{http://#1}{\nolinkurl{#1}}}

\begin{document}

\twocolumn[
\aistatstitle{Contextual Directed Acyclic Graphs}
\aistatsauthor{Ryan Thompson \And Edwin V. Bonilla \And Robert Kohn}
\aistatsaddress{University of New South Wales \\ CSIRO's Data61 \And CSIRO's Data61 \And University of New South Wales}
]

\begin{abstract}
Estimating the structure of directed acyclic graphs (DAGs) from observational data remains a significant challenge in machine learning. Most research in this area concentrates on learning a single DAG for the entire population. This paper considers an alternative setting where the graph structure varies across individuals based on available ``contextual'' features. We tackle this contextual DAG problem via a neural network that maps the contextual features to a DAG, represented as a weighted adjacency matrix. The neural network is equipped with a novel projection layer that ensures the output matrices are sparse and satisfy a recently developed characterization of acyclicity. We devise a scalable computational framework for learning contextual DAGs and provide a convergence guarantee and an analytical gradient for backpropagating through the projection layer. Our experiments suggest that the new approach can recover the true context-specific graph where existing approaches fail.
\end{abstract}

\section{INTRODUCTION}
\label{sec:introuction}

Directed acyclic graphs (DAGs)---graphs with directed edges and no cycles---are a core tool for probabilistic graphical modeling \parencite[see, e.g.,][]{Murphy2023}. Their ability to represent complex multivariate relationships makes them valuable across a broad range of domains, including psychology \parencite{Foster2010}, economics \parencite{Imbens2020}, and epidemiology \parencite{Tennant2021}. Though DAGs have a long and rich history in machine learning and statistics \parencite{Lauritzen1988,Pearl1988}, they have recently attracted significant attention due to a computational breakthrough by \textcite{Zheng2018} that reframed the combinatorial DAG structure learning problem as a continuous optimization problem. This and subsequent continuous reformulations of the DAG problem have enabled scalable structure learning for applications that were considered intractable. See \textcite{Vowels2022} for a recent review of this work.

A major focus of the structure learning literature, both recent and past, is on estimating a single DAG for the entire population. Although these ``fixed DAGs'' are suitable in standard settings, they can fail in heterogeneous or non-stationary environments \parencite{Huang2020,Zhou2022} and in the presence of external modifying effects \parencite{Ni2019}. An important application where these types of problems arise is that of understanding recreational drug consumption patterns. The consumption patterns can be represented as a DAG whose structure is dictated according to individual personality traits. As traits differ across individuals, so does the graph structure---an edge between a pair of drugs present for one individual might be absent for another. Furthermore, an edge may point in opposite directions for different individuals. The personality traits in this example represent information that encodes context into the graph. DAGs that change with context, which we refer to as ``contextual DAGs,'' lie beyond the capability of existing structure learning methods.

\begin{figure*}[t]
\centering
\begin{tabular}{cccccc}
\begin{subfigure}[t]{0.14\linewidth}
\centering
\caption*{\footnotesize True DAG}
\vspace{0.2cm}
\resizebox{0.85\linewidth}{!}{\includegraphics{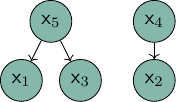}}
\end{subfigure}
&
\begin{subfigure}[t]{0.14\linewidth}
\centering
\caption*{\footnotesize Contextual DAG }
\vspace{0.2cm}
\resizebox{0.85\linewidth}{!}{\includegraphics{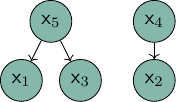}}
\end{subfigure}
&
\begin{subfigure}[t]{0.14\linewidth}
\centering
\caption*{\footnotesize Fixed DAG}
\vspace{0.2cm}
\resizebox{0.85\linewidth}{!}{\includegraphics{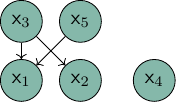}}
\end{subfigure}
&
\begin{subfigure}[t]{0.14\linewidth}
\centering
\caption*{\footnotesize True DAG}
\vspace{0.2cm}
\resizebox{\linewidth}{!}{\includegraphics{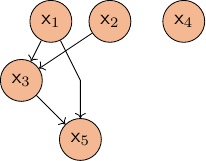}}
\end{subfigure}
&
\begin{subfigure}[t]{0.14\linewidth}
\centering
\caption*{\footnotesize Contextual DAG }
\vspace{0.2cm}
\resizebox{\linewidth}{!}{\includegraphics{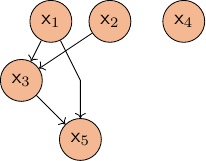}}
\end{subfigure}
&
\begin{subfigure}[t]{0.14\linewidth}
\centering
\caption*{\footnotesize Fixed DAG}
\vspace{0.2cm}
\resizebox{0.85\linewidth}{!}{\includegraphics{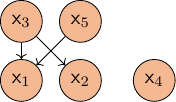}}
\end{subfigure} \\~\\
\multicolumn{3}{c}{\small (a) Graph conditional on $z=z_1$.} & \multicolumn{3}{c}{\small (b) Graph conditional on $z=z_2$.}
\end{tabular}
\caption{Illustration of the contextual DAG. The true graph is a function of the features $z$. The left three graphs correspond to one realization of $z$, and the right three correspond to a second independent realization.}
\label{fig:example}
\end{figure*}

Our paper demonstrates that it is possible to learn contextual DAGs using available contextual information. Let $x=(\mathsf{x}_1,\ldots,\mathsf{x}_p)^\top$ be a vector of continuous variables (nodes) and $z=(\mathsf{z}_1,\ldots,\mathsf{z}_m)^\top$ be a vector of contextual features. We study linear DAGs on $x$, which can be represented via a $p\times p$ weighted adjacency matrix $W=(\mathsf{w}_{jk})$ such that $\mathsf{w}_{jk}$ is non-zero if and only if a directed edge exists from node $j$ to node $k$. We define a fixed DAG as a graph such that $W$ is independent of $z$ and a contextual DAG as a graph where $W$ is a function of $z$. Via the weighted adjacency representation, a contextual DAG can be expressed in terms of a (linear) structural equation model:
\begin{equation}
\label{eq:sem}
\mathsf{x}_k=\sum_{j=1}^p\mathsf{x}_j\mathsf{w}_{jk}(z)+\svarepsilon_k,\quad k=1,\ldots,p,
\end{equation}
where $\varepsilon=(\svarepsilon_1,\ldots,\svarepsilon_p)^\top$ is a zero mean stochastic noise vector. Given the contextual features $z$, the weighted adjacency function $W(z):=[\mathsf{w}_{jk}(z)]$ encodes the parents of $\mathsf{x}_k$ (the $j$s such that $\mathsf{w}_{jk}(z)$ is non-zero) and the effects of those parents. A key property of a DAG is that its nodes can be sorted such that parents come before children, known as a topological ordering \parencite[see, e.g.,][\S 4.2]{Murphy2023}. Thus, a contextual DAG allows both the topological ordering and weights to be functions of the contextual features. A fixed DAG is a special case where these functions are constant.

Estimating the structural equation model \eqref{eq:sem} is challenging, even in the fixed case, because of the combinatorial nature of the search space. Leveraging recent breakthroughs in continuous optimization routines for fixed DAGs \parencite{Zheng2018,Bello2022}, we devise a neural network architecture that learns to map from the contextual feature space to the space of DAGs. It achieves this task using a novel layer that projects a simple directed graph\footnote{A simple directed graph is a directed graph with no self-loops (an edge directed from a node back to itself) and at most one edge in each direction between any pair of nodes.} from a feedforward neural network onto the space of acyclic graphs. This ``projection layer'' simultaneously induces sparsity over the graph by constraining it to an $\ell_1$ ball of appropriate size, leading to a parsimonious representation of the data with as few edges as feasible. Although the projection layer constitutes an iterative optimization algorithm, we establish its convergence properties and show that it can be solved in parallel on a GPU, allowing for efficient forward passes. We also derive an analytical form for the layer's gradients to facilitate efficient backpropagation during training.

To briefly illustrate our proposal, we consider a small-scale Gaussian model with five variables generated by a graph whose node ordering depends on a contextual feature vector $z\in\mathbb{R}^2$. The number of edges also varies with the contextual features, so graphs for different realizations of $z$ need not have the same sparsity. Figure~\ref{fig:example} compares the true DAG under this model with the contextual and fixed DAG estimates for two independent draws of $z$. The topological ordering of the true graph under $z_1$ is distinct from that under $z_2$. For instance, the direction of the edge between $\mathsf{x}_1$ and $\mathsf{x}_5$ changes from $z_1$ to $z_2$. The number of edges also varies---three edges under $z_1$ and four under $z_2$. By learning to predict a graph from the contextual features, the contextual DAG captures these evolving dependencies and recovers the true graph in both cases. Conversely, the fixed DAG, which assumes the structure is unchanging, fails to recover either graph.

This paper makes four contributions:
\begin{enumerate}
\item We introduce a new neural network architecture that learns to predict context-specific DAGs.
\item We establish a theoretical convergence guarantee for the network's projection layer and derive its analytical gradient, avoiding the expense of automatically differentiating through the optimizer.
\item We analyze our approach in numerical experiments that show it is the best available tool for structure learning in the contextual setting.
\item We provide the structure learning community with \texttt{ContextualDAG}, an open-source Julia implementation that scales well on GPUs.
\end{enumerate}

Before proceeding, we remark that we do not seek to causally interpret contextual DAGs in this paper, and hence make no causal assumptions (e.g., causal sufficiency). We instead focus on the core learning problem of finding a contextual DAG that fits the data well.

\section{CONTEXTUAL DAGS}
\label{sec:contextual}

\begin{figure*}[t]
\centering
\resizebox{0.8\linewidth}{!}{\includegraphics{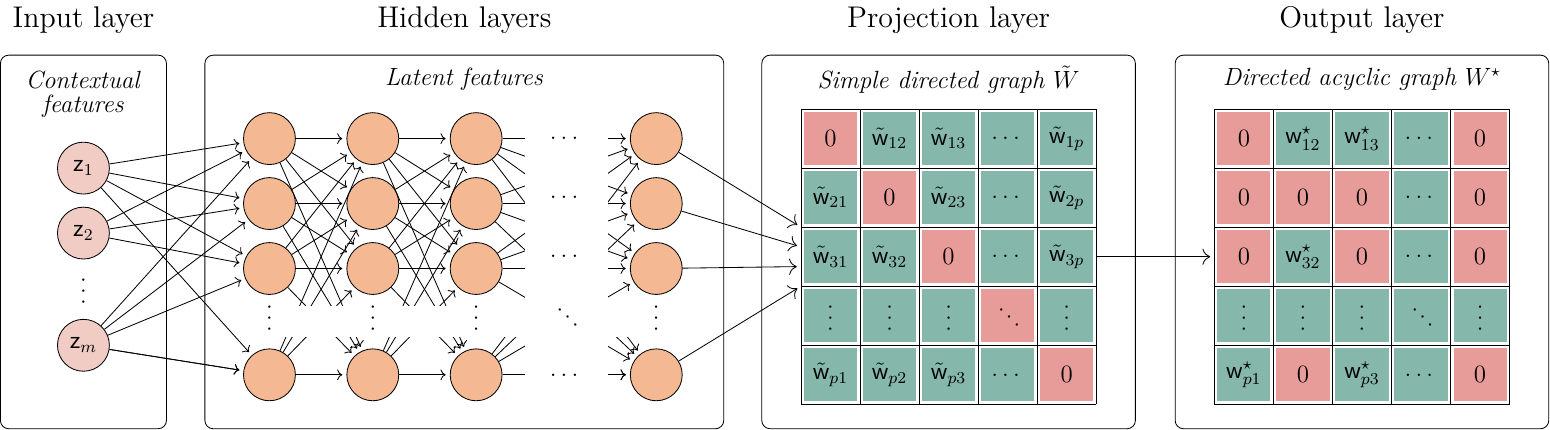}}
\caption{Neural network architecture of the contextual DAG. The features $z$ pass through hidden layers to produce a simple directed graph $\tilde{W}$. The projection layer makes $\tilde{W}$ acyclic and sparse, resulting in a DAG $W^\star$.}
\label{fig:architecture}
\end{figure*}

To facilitate exposition, it is helpful to rewrite the structural equation model \eqref{eq:sem} more compactly as
\begin{equation}
\label{eq:sem2}
x=W(z)^\top x+\varepsilon,
\end{equation}
where the function $W(z):\mathbb{R}^m\to\mathrm{DAG}$ maps the contextual features $z$ onto the space of acyclic weighted adjacency matrices. More precisely, the sparsity pattern of the matrix produced by $W(z)$, corresponding to the set of edges, changes with $z$ under the constraint that the pattern never contains a directed cycle.

\subsection{Fixed DAG Problem}

Before considering the estimation problem corresponding to the model \eqref{eq:sem2}, it is useful to review the learning problem for a fixed DAG where $z$ is absent:
\begin{equation}
\label{eq:fixed}
\begin{split}
\min_{W\in\mathbb{R}^{p\times p}}\quad&\operatorname{E}_{x}\left[\|x-W^\top x\|_2^2\right] \\
\operatorname{s.t.}\quad&W\in\mathrm{DAG} \\
&\|W\|_{\ell_1}\leq\lambda.
\end{split}
\end{equation}
Here, $\|W\|_{\ell_1}:=\|\operatorname{vec}(W)\|_1$ is the sum of absolute values of the elements of $W$. The first constraint in \eqref{eq:fixed} states that $W$ cannot contain cycles (i.e., it must be acyclic), while the second constraint is an $\ell_1$ regularizer that promotes sparsity in $W$. A sparse $W$ implies a parsimonious graph with few edges.

\subsection{Contextual DAG Problem}

In the contextual setting, $W$ is no longer a fixed matrix but instead a function of the random variable $z$:
\begin{equation}
\label{eq:population}
\begin{split}
\min_{W\in\mathcal{F}}\quad&\operatorname{E}_{x,z}\left[\|x-W(z)^\top x\|_2^2\right] \\
\operatorname{s.t.}\quad&W(z)\in\mathrm{DAG}\text{ for all }z\in\mathbb{R}^m \\
&\operatorname{E}_z\left[\|W(z)\|_{\ell_1}\right]\leq\lambda,
\end{split}
\end{equation}
where the set $\mathcal{F}$ is some class of functions that define feasible solutions (e.g., neural networks). The DAG constraint is now a statement that the (random) matrix $W(z)$ must correspond to a DAG over the entire support of $z$. Said differently, for all realizations of $z$, the weighted adjacency matrix must be acyclic. The $\ell_1$ regularizer continues to impose sparsity, but now on the expectation of $W(z)$, meaning that the graph's sparsity can change. It may contain fewer edges for some realizations of $z$ and more edges for other $z$.

Given a sample of observations $(x_i,z_i)_{i=1}^n$, the data version of \eqref{eq:population} replaces the population expectations with their sample counterparts:
\begin{equation*}
\begin{split}
\min_{W\in\mathcal{F}}\quad&\frac{1}{n}\sum_{i=1}^n\|x_i-W(z_i)^\top x_i\|_2^2 \\
\operatorname{s.t.}\quad&W(z)\in\mathrm{DAG}\text{ for all }z\in\mathbb{R}^m \\
&\frac{1}{n}\sum_{i=1}^n\|W(z_i)\|_{\ell_1}\leq\lambda.
\end{split}
\end{equation*}
To realize this estimator, we take the function class $\mathcal{F}=\{W_\theta(z):\theta\in\mathbb{R}^d\}$ with $W_\theta(z)$ a certain architecture of feedforward neural networks parameterized by weights $\theta$. This choice results in our proposal:
\begin{equation}
\label{eq:sample}
\begin{split}
\min_{\theta\in\mathbb{R}^d}\quad&\frac{1}{n}\sum_{i=1}^n\|x_i-W_\theta(z_i)^\top x_i\|_2^2 \\
\operatorname{s.t.}\quad&W_\theta(z)\in\mathrm{DAG}\text{ for all }z\in\mathbb{R}^m \\
&\frac{1}{n}\sum_{i=1}^n\|W_\theta(z_i)\|_{\ell_1}\leq\lambda.
\end{split}
\end{equation}
Next, we present a novel neural network architecture that restricts the network's co-domain to the constraint set of \eqref{eq:sample}, a difficult task in general.

\subsection{Neural Network Architecture}

Figure~\ref{fig:architecture} illustrates our neural network architecture. The network has two main components: hidden layers and a projection layer. The hidden layers are standard linear transformations followed by non-linear activation functions (e.g., rectified linear units), except for the final hidden layer, which has no activation. The purpose of the hidden layers is to capture any non-linear effects of the contextual features. Their output is a simple directed graph in the form of a $p\times p$ weighted adjacency matrix, denoted $\tilde{W}$. Recall that a simple directed graph has directed edges but no self-loops, so the diagonal of $\tilde{W}$ is zero. Hence, the final hidden layer has $p\times(p-1)$ neurons. The network's second major component---the projection layer---is required because it is impossible to obtain an acyclic or sparse graph using only the hidden layers.

\subsection{Projection Layer}

The role of the projection layer is to transform the (dense) simple directed graph $\tilde{W}$ into a (sparse) DAG $W^\star$. During training, $n$ weighted adjacency matrices $\tilde{W}_1,\ldots,\tilde{W}_n$ are projected onto the feasible set each time a forward pass is performed. This projection is a solution to a constrained optimization problem:
\begin{equation}
\label{eq:proj}
\begin{split}
\underset{W_1,\ldots,W_n\in\mathbb{R}^{p\times p}}{\min}\quad&\frac{1}{2}\sum_{i=1}^n\|\tilde{W}_i-W_i\|_F^2 \\
\operatorname{s.t.}\quad & W_i\in\mathrm{DAG},\quad i=1,\ldots,n \\
&\frac{1}{n}\sum_{i=1}^n\|W_i\|_{\ell_1}\leq\lambda.
\end{split}
\end{equation}
The subscript $F$ in the objective function denotes the Frobenius norm. Due to the combinatorial DAG constraint, the feasible set of \eqref{eq:proj} is non-convex. The number of possible DAGs on $p$ nodes is super-exponential in $p$, limiting the scalability of exact combinatorial approaches. In our setting, where the projection operates on $n$ weighted adjacency matrices at every forward pass, scalability issues are only further exacerbated.

Recently, \textcite{Zheng2018}, \textcite{Bello2022}, and others proposed continuous characterizations of acyclicity in the form of $h(W)=0$ for some function $h(W)$ having the property
\begin{equation*}
h(W)=0\Longleftrightarrow W\in\mathrm{DAG}.
\end{equation*}
In particular, \textcite{Bello2022} showed that for all matrices $W$ whose element-wise square $W\circ W$ has spectral radius less than $s>0$, a well-suited choice of the function $h(W)$ is
\begin{equation}
\label{eq:hs}
h_s(W):=-\log\,\det(sI-W\circ W)+p\log(s).
\end{equation}
This characterization of acyclicity, known as DAGMA, allows us to replace the combinatorial DAG constraint with its equivalent log determinant characterization:
\begin{equation}
\label{eq:project}
\begin{split}
\underset{W_1,\ldots,W_n\in\mathbb{W}^s}{\min}\quad&\frac{1}{2}\sum_{i=1}^n\|\tilde{W}_i-W_i\|_F^2 \\
\operatorname{s.t.}\quad & h_s(W_i)=0,\quad i=1,\ldots,n \\
&\frac{1}{n}\sum_{i=1}^n\|W_i\|_{\ell_1}\leq\lambda,
\end{split}
\end{equation}
where $\mathbb{W}^s:=\{W\in\mathbb{R}^{p\times p}:\rho(W\circ W)<s\}$ and $\rho(\cdot)$ is the matrix's spectral radius. Though \eqref{eq:project} remains a non-convex problem, the constraint function $h_s(W)$ is continuous and differentiable, enabling applications of scalable first-order optimization methods.

\section{PROJECTION ALGORITHMS}
\label{sec:projection}

A forward pass through the neural network performs the projection \eqref{eq:project} onto the intersection of a non-convex set (the $\log\,\det$ level set) and a convex set (the $\ell_1$ ball). We split this projection into two steps for efficient computation: (1) projection onto the $\log\,\det$ level set and (2) projection onto the $\ell_1$ ball. Proposition~\ref{prop:proj} states that this approach yields a feasible solution to the original problem. Its proof is in Appendix~\ref{app:proj}.
\begin{proposition}
\label{prop:proj}
Let $\tilde{W}\in\mathbb{R}^{p\times p}$, $\hat{W}$ be the projection of $\tilde{W}$ onto the $\log\,\det$ level set, and $W^\star$ the projection of $\hat{W}$ onto the $\ell_1$ ball. Then $W^\star$ lies on the intersection of the $\log\,\det$ level set and the $\ell_1$ ball.
\end{proposition}

\subsection{\texorpdfstring{$\log\,\det$}{log det} Projection}

The optimization problem of the $\log\,\det$ projection readily separates into $n$ identical subproblems. We therefore focus on this subproblem for simplicity:
\begin{equation}
\label{eq:dag}
\begin{split}
\underset{W\in\mathbb{W}^s}{\min}\quad&\frac{1}{2}\|\tilde{W}-W\|_F^2 \\
\operatorname{s.t.}\quad & h_s(W)=0.
\end{split}
\end{equation}
The projection \eqref{eq:dag} is a complex problem for which no analytical solution is available. \textcite{Bello2022} devised a path-following algorithm for general $\log\,\det$ constrained problems that shifts the constraint function to the objective and then treats the original loss function as a penalty that is progressively annealed towards zero. In our setting, the objective function corresponding to this approach is
\begin{equation}
\label{eq:path}
f_{\mu,s}(W;\tilde{W}):=\frac{\mu}{2}\|\tilde{W}-W\|_F^2+h_s(W),
\end{equation}
where the coefficient $\mu>0$ is successively decreased along the path. Lemma 6 of \textcite{Bello2022} states that the limiting solution as $\mu\to0$ satisfies the original $\log\,\det$ constraint.

Algorithm~\ref{alg:logdet} uses the formulation \eqref{eq:path} to provide a method for projecting onto the $\log\,\det$ level set.

\begingroup
\removelatexerror
\begin{algorithm}[H]
\small
\caption{$\log\,\det$ projection}
\label{alg:logdet}
\SetKwInOut{Input}{input}
\SetKwInOut{Output}{output}
\Input{Adjacency matrix $\tilde{W}\in\mathbb{R}^{p\times p}$, $\log\,\det$ parameter $s>0$, path coefficient $\mu>0$, decay factor $\alpha\in(0,1)$, step count $T\in\mathbb{N}$}
Set $W^{(0)}\gets0$ \\
\For{$t=0,1,\ldots,T-1$}{
Initialize $W$ at $W^{(t)}$ and solve
\begin{equation}
\label{eq:inner}
W^{(t+1)}\gets\underset{W\in\mathbb{W}^s}{\arg\,\min}\quad f_{\mu,s}(W;\tilde{W})
\end{equation}
Set $\mu\gets\alpha\mu$
}
\Output{Adjacency matrix $\hat{W}=W^{(T)}$}
\end{algorithm}
\endgroup

In the experiments, we set the path coefficient $\mu=1$, decay factor $\alpha=0.5$, and step count $T=10$. These default values generally work well in our experience.

The performance of Algorithm~\ref{alg:logdet} depends critically on efficiently solving the inner optimization problem \eqref{eq:inner} at each step along the path, i.e., minimizing the objective function \eqref{eq:path} for a fixed value of $\mu$. Recall that the gradient of $\log\det(X)$ is $X^{-\top}$, so that
\begin{equation*}
\nabla h_s(W)=2(sI-W\circ W)^{-\top}\circ W,
\end{equation*}
where $h_s(W)$ is defined in \eqref{eq:hs}. The gradient $\nabla f_{\mu,s}(W;\tilde{W})$ of the objective function immediately follows. Using these gradients, the inner optimization problem can be solved by first-order methods such as gradient descent. To characterize the behavior of gradient descent on this problem, Theorem~\ref{thm:converge} presents its convergence properties. The proof is in Appendix~\ref{app:converge}.
\begin{theorem}
\label{thm:converge}
Let $W^{(0)}=0$, $\tilde{W}\in\mathbb{R}^{p\times p}$ with $|\tilde{\mathsf{w}}_{jk}|\leq1$, $s\geq1+\max(\|\tilde{W}\|_1,\|\tilde{W}\|_\infty)$, and $\mu>0$. Define the gradient descent update
\begin{equation*}
W^{(k+1)}=W^{(k)}-\frac{1}{\bar{c}}\nabla f_{\mu,s}(W^{(k)};\tilde{W}).
\end{equation*}
Then, for any $\bar{c}\geq c=\max(\mu/2,2\sqrt{p}+4p\|\tilde{W}\|_F)$, the sequence $\{f_{\mu,s}(W^{(k)};\tilde{W})\}_{k\in\mathbb{N}}$ decreases, converges, and satisfies the inequality
\begin{equation*}
\begin{split}
&f_{\mu,s}(W^{(k)};\tilde{W})-f_{\mu,s}(W^{(k+1)};\tilde{W})\geq \\
&\hspace{4cm}\frac{\bar{c}-c}{2}\|W^{(k+1)}-W^{(k)}\|_F^2.
\end{split}
\end{equation*}
\end{theorem}

Theorem~\ref{thm:converge} establishes that gradient descent applied to the inner optimization problem \eqref{eq:inner} generates a convergent sequence of objective values. This result is not trivial to prove and relies on showing Lipschitz continuity of $f_{\mu,s}(W;\tilde{W})$ under the theorem's conditions. The condition $|\tilde{\mathsf{w}}_{jk}|\leq1$ is readily satisfied by scaling the algorithm's inputs (and rescaling its outputs).

The step size $1/\bar{c}$ implied by Theorem~\ref{thm:converge} can be quite small in practice, slowing convergence. Likewise, the implied $\log\,\det$ parameter $s$ can be quite large, also slowing convergence. It is unclear if these conservative values are an artifact of the proof strategy. Our numerical experience is that a more aggressive step size with $\bar{c}=p$ and $s=1$ typically leads to convergence.

\subsection{\texorpdfstring{$\ell_1$}{l1} Projection}

Let $\hat{W}_1,\ldots,\hat{W}_n$ be the output of the $\log\,\det$ projection. With these matrices, the $\ell_1$ projection solves
\begin{equation}
\label{eq:l1}
\begin{split}
\underset{W_1,\ldots,W_n\in\mathbb{R}^{p\times p}}{\min}\quad&\frac{1}{2}\sum_{i=1}^n\|\hat{W}_i-W_i\|_F^2 \\
\operatorname{s.t.}\quad &\frac{1}{n}\sum_{i=1}^n\|W_i\|_{\ell_1}\leq\lambda.
\end{split}
\end{equation}
It is straightforward to show that the projection \eqref{eq:l1} is solved by thresholding $\hat{W}_1,\ldots,\hat{W}_n$ element-wise using
\begin{equation*}
S_\kappa(\mathsf{w}):=\operatorname{sign}(\mathsf{w})\max(|\mathsf{w}|-\kappa,0),
\end{equation*}
where the threshold $\kappa\geq0$ is derived from $\hat{W}_1,\ldots,\hat{W}_n$. \textcite{Duchi2008} provide a non-iterative algorithm involving only a few low-complexity operations for computing $\kappa$ in the case of vector arguments. As described in Appendix~\ref{app:l1}, their algorithm readily extends to matrix arguments. The computation of $\kappa$ is performed only during training. For inference, the $\kappa$ from the training set is used. This projection yields the neural network's final output $W_1^\star,\ldots,W_n^\star$.

\section{SCALABLE COMPUTATION}
\label{sec:scalable}

\begin{figure*}[t]
\centering
\begin{subfigure}[t]{0.4\linewidth}
\centering
\includegraphics[width=\linewidth]{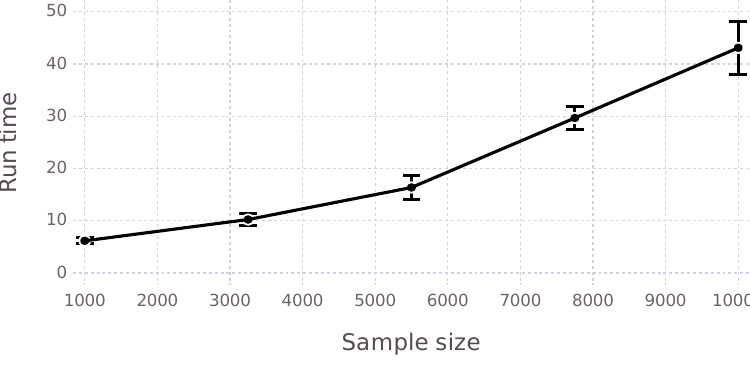}
\end{subfigure}
\begin{subfigure}[t]{0.4\linewidth}
\centering
\includegraphics[width=\linewidth]{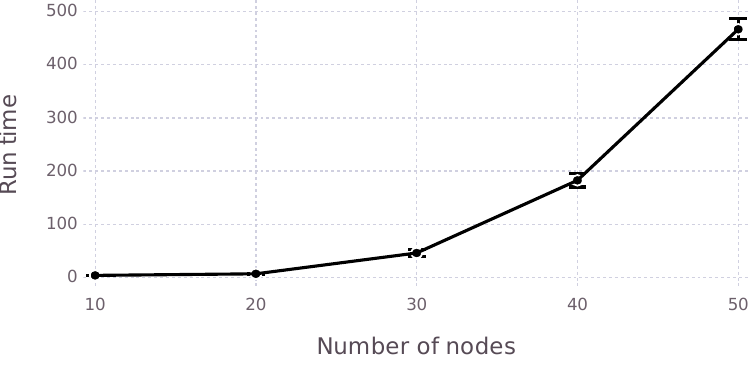}
\end{subfigure}
\caption{Run times in seconds for 10 epochs of a contextual DAG on an NVIDIA RTX 4090 over 10 synthetic datasets. The number of nodes $p=20$ in the left plot and the sample size $n=1000$ in the right plot. The number of contextual features $m=2$. The solid points are averages and the error bars are one standard errors.}
\label{fig:timings}
\end{figure*}

Even with a first-order method for the projection layer, scaling contextual DAGs to large sample sizes requires careful treatment of the forward and backward passes through the network. We now propose some novel, efficient methods and evaluate their scalability.

\subsection{Forward Pass}

At each forward pass during training, the projection layer acts on all $n$ observations in the sample. These projections, in turn, each involve inverting a $p\times p$ matrix at every gradient descent iteration, which has complexity $O(p^3)$. Though cubic complexity is typical of existing structure learning methods \parencite[e.g.,][]{Zheng2018,Bello2022}, it makes it prohibitive to carry out the projections in sequence, as a naive implementation might do. For this reason, we implement Algorithm~\ref{alg:logdet} as a batched solver that handles all $n$ problems in parallel. Batched BLAS \parencite{Dongarra2017} and its CUDA equivalent expose batched linear algebra routines and, critically, batched linear algebra solvers to perform parallel matrix inversions. Via these routines, our implementation can perform forward passes in reasonable run times for $n$ in the order of thousands or tens of thousands.

\subsection{Backward Pass}

The second consideration during training is the backward pass, where gradients are propagated backward through the network. The presence of the projection layer complicates this step because we require the gradients of the output $W_1^\star,\ldots,W_n^\star$ from this layer with respect to the input $\tilde{W}_1,\ldots,\tilde{W}_n$ to backpropagate successfully. Despite the layer involving a complex optimization problem, Theorem~\ref{thm:gradient} gives an analytical expression for its gradients. The proof is in Appendix~\ref{app:gradient}.
\begin{theorem}
\label{thm:gradient}
Let $\tilde{W}_1,\ldots,\tilde{W}_n\in\mathbb{R}^{p\times p}$. Define $W_1^\star,\ldots,W_n^\star\in\mathbb{R}^{p\times p}$ as a solution to the projection \eqref{eq:project}. Let $\mathcal{A}:=\{(i,j,k):\mathsf{w}^\star_{ijk}\neq0\}$ be the active set of edges, where $\mathsf{w}_{ijk}^\star$ is the element of $W_i^\star$ at the $j$th row and $k$th column. Then, if the $\ell_1$ constraint is binding (i.e., $n^{-1}\sum_{i=1}^n\|W_i^\star\|_{\ell_1}=\lambda$), the gradient of $\mathsf{w}_{ijk}^\star$ with respect to $\tilde{\mathsf{w}}_{lqr}$ is given by
\begin{equation*}
\begin{split}
\frac{\partial\mathsf{w}_{ijk}^\star}{\partial\tilde{\mathsf{w}}_{lqr}}=
\begin{dcases}
\delta_{ijk}^{lqr}-\frac{\operatorname{sgn}(\mathsf{w}_{lqr}^\star)\operatorname{sgn}(\mathsf{w}_{ijk}^\star)}{|\mathcal{A}|} & \text{if }(i,j,k)\in\mathcal{A} \\
0 & \text{otherwise},
\end{dcases}
\end{split}
\end{equation*}
where $\delta_{ijk}^{lqr}:=1[(i,j,k)=(l,q,r)]$ and $|\mathcal{A}|$ is the cardinality of $\mathcal{A}$. If the $\ell_1$ constraint is not binding, the gradient is given by
\begin{equation*}
\frac{\partial\mathsf{w}_{ijk}^\star}{\partial\tilde{\mathsf{w}}_{lqr}}=
\begin{cases}
\delta_{ijk}^{lqr} & \text{if }(i,j,k)\in\mathcal{A} \\
0 & \text{otherwise}.
\end{cases}
\end{equation*}
\end{theorem}
Theorem~\ref{thm:gradient} states that the elements of $W_1^\star,\ldots,W_n^\star$ that are zero (i.e., the non-existent edges) have gradient zero, while the gradients of the remaining non-zero elements depend on their respective signs. The dependence vanishes as the number of non-zero elements grows. Importantly, the gradient only requires knowledge of the projection layer's output $W_1^\star,\ldots,W_n^\star$. Once the forward pass produces these outputs, the gradient needed for the backward pass comes with virtually no additional computation. The alternative, which is to automatically differentiate through the layer's gradient descent routine, is exceptionally expensive due the algorithm's iterative nature.

\subsection{Complexity}

With a fixed number of hidden-layer neurons, the forward or backward pass through the hidden layers, which input $m$ contextual features and output a $p\times p$ matrix, takes $O(nm+np^2)$ operations. A forward pass through the projection layer takes $O(np^3)$ operations, while the backward pass only takes $O(np^2)$ operations (based on Theorem~\ref{thm:gradient}). The total complexity of training the network is $O(nm+np^3)$. Hence, the training time is linear in the sample size $n$ and the number of contextual features $m$, and cubic in the number of nodes $p$. We emphasize again that cubic complexity on $p$ is typical of continuous structure learning algorithms. Figure~\ref{fig:timings} plots the average run times for training a contextual DAG as a function of $n$ and $p$. The run time as a function of $n$ is indeed fairly linear. The run time as a function of $p$ is non-linear but not quite cubic. Of course, cubic complexity is worst-case.

\subsection{Implementation}

Our algorithmic framework is provided in the \texttt{Julia} \parencite{Bezanson2017} implementation \texttt{ContextualDAG} using the deep learning library \texttt{Flux} \parencite{Innes2018}. We devise a custom pathwise optimization strategy to speed up training, described in Appendix~\ref{app:pathwise}. \texttt{ContextualDAG} is available at
\begin{center}
\rurl{github.com/ryan-thompson/ContextualDAG.jl}.
\end{center}

\section{RELATED WORK}
\label{sec:related}

To the best of our knowledge, no existing methods are available for learning DAGs with varying structure as general as our proposal. The closest work to ours is \textcite{Ni2019}, whose approach also allows some components of the DAG structure to vary as a function of external features, albeit with two crucial differences. First, the approach in \textcite{Ni2019} assumes the ordering of the nodes is fixed and known by the user. In contrast, our approach implicitly learns this ordering from the data directly and does not assume it is fixed. Second, \textcite{Ni2019} learn the map from the contextual features to the DAG using splines rather than using a neural network as we do. Though splines are computationally appealing, they do not generalize well beyond low-complexity function classes (i.e., smooth functions), which can lead them to underperform relative to neural networks in practice \parencite[see, e.g.,][]{Agarwal2021}.

Several other papers have explored ideas that relate more generally to contextual DAGs. \textcite{Boutilier1996} studied the notion of ``context-specific independence'' wherein dependencies between nodes in a DAG (of fixed topological ordering) can depend on the states of other nodes. \textcite{Geiger1996} considered a related class of models---similarity networks and multinets---which also represent varying independencies. \textcite{Oates2016} devised a method for learning subject-specific graphs whereby multiple graphs are trained simultaneously under a regularizer that penalizes distances between the graphs. A dependency network is used to describe the relations between subjects. In another line of work, \textcite{Ahmed2022} proposed semantic probabilistic layers for neural networks, which could be applied to impose constraints such as acyclicity to produce context-specific DAGs. However, their approach requires the true structure during training, which is unavailable in our setting.

Undirected graphical models with varying structure have received more attention than their directed counterparts. See the recent papers \textcite{Ni2022,Niu2023,Zhang2023}, and references therein. Whereas we constrain our model's co-domain to acyclic matrices, undirected graphical models that vary with external features typically constrain the co-domain to positive definite matrices. This constraint is comparatively straightforward to handle as it does not require iterative algorithms that are necessary for enforcing acyclicity. These works also focus on low-complexity (linear or piecewise linear) maps from the external features to graphs, restricting their expressiveness. We are unaware of any existing work on graphical models that employ neural networks for this task. However, \textcite{Thompson2023} showed that neural networks can successfully incorporate external features into non-graphical (sparse linear) models.

More broadly, our paper complements the literature on fixed DAG learning. \textcite{Zheng2018} introduced the NOTEARS framework, which was the first to provide a continuous characterization of the DAG constraint. Subsequent work, including DAGMA \parencite{Bello2022}---used in this paper---as well as \textcite{Ng2020,Yu2021,Gillot2022} refined the continuous optimization approach. See also the extensions and applications in \textcite{Yu2019,Lachapelle2020,Pamfil2020,Geffner2022,Gong2023}. The appeal of continuous algorithms is that they scale gracefully to large graphs and can be augmented with more complex combinatoric approaches if needed \parencite{Manzour2021,Deng2023}. Our work sits parallel to this line of research. We build a neural network architecture around these algorithms to enable the otherwise fixed structure of DAGs to vary flexibly with contextual features.

\section{SYNTHETIC EXPERIMENTS}
\label{sec:synthetic}

\begin{figure*}[t]
\centering
\begin{subfigure}[t]{\linewidth}
\centering
\includegraphics[width=0.8\linewidth]{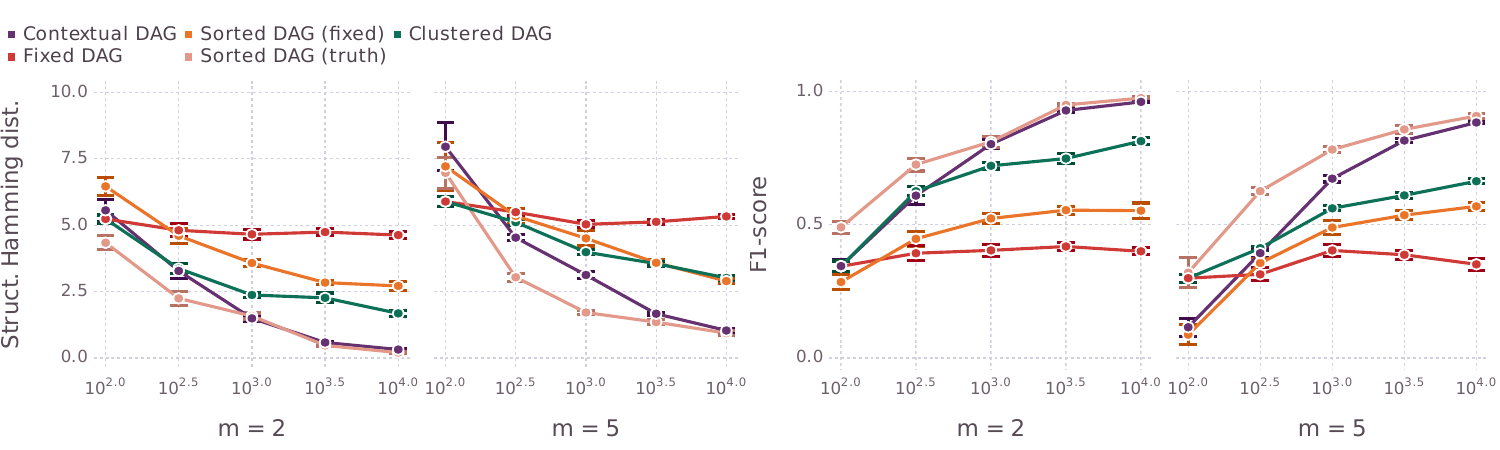}
\caption{Structure recovery as a function of the sample size $n$ for $m\in\{2,5\}$ contextual features.}
\label{fig:erdos-renyi-n}
\end{subfigure}
\begin{subfigure}[t]{\linewidth}
\centering
\includegraphics[width=0.8\linewidth]{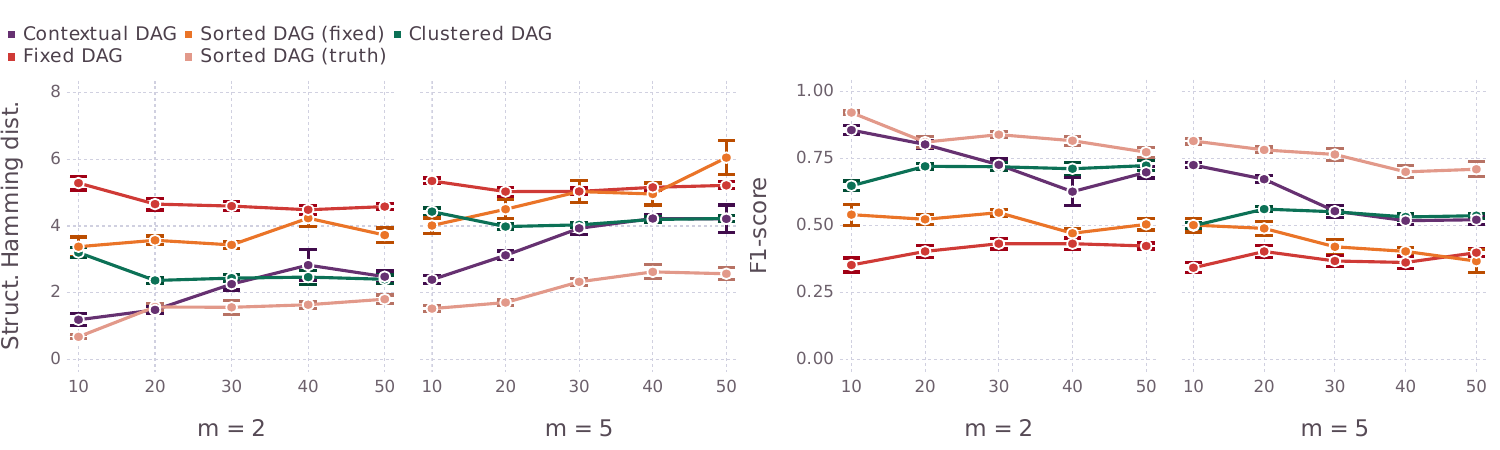}
\caption{Structure recovery as a function of the number of nodes $p$ for $m\in\{2,5\}$ contextual features.}
\label{fig:erdos-renyi-p}
\end{subfigure}\caption{Structure recovery performance on varying Erdős-Rényi graphs over 10 synthetic datasets. The number of nodes $p=20$ in the top row and the sample size $n=1000$ in the bottom row. The solid points are averages and the error bars are one standard errors. The sorted DAG (truth) uses the ground truth topological order.}
\label{fig:erdos-renyi}
\end{figure*}

Here, we evaluate our approach on synthetic data. We compare against a fixed DAG that estimates a single graph over the entire sample. We also include two ``sorted DAGs'' that allow the graph structure to vary but assume the ordering of the nodes is given. The first uses a fixed topological order \parencite[as in][]{Ni2019} taken from the estimated fixed DAG. The second uses the \textit{true} varying topological order, representing an idealized case and a bound on what our approach can achieve. The two sorted DAGs employ the same architecture as the contextual DAG but use binary masking matrices rather than an acyclicity projection to encode the topological orderings. Finally, we include a ``clustered DAG'' that assigns observations on $z$ into $\lceil n/100\rceil$ clusters using $k$-means and fits a fixed DAG to each cluster. All approaches use the DAGMA acyclicity characterization of \textcite{Bello2022}. Appendix~\ref{app:implementation} provides the details of their implementation.

 We generate synthetic datasets according to the structural equation model \eqref{eq:sem2} using Erdős-Rényi and scale-free graphs where the edge weights, node ordering, and number of edges vary with the contextual features. Appendix~\ref{app:synthetic} details the full simulation design. As structure recovery metrics, we measure the structural Hamming distance and F1-score averaged over all the graphs predicted for a testing set generated independently and identically to the training set.

Figure~\ref{fig:erdos-renyi-n} reports the results as function of $n$ for the Erdős-Rényi graphs. Besides the true sorted DAG, the contextual DAG is the only method to recover the context-specific graphs accurately as $n$ grows large. Its neural network--enabled by the projection layer---is increasingly able to well-approximate the true function. The gap between the contextual DAG and true sorted DAG, representing the cost of not knowing the topological order a priori, vanishes with growing $n$.

The fixed sorted DAG, which allows the edge weights and sparsity of the graph to change but not the topological ordering, also improves with $n$ but at a slower rate. However, its performance eventually plateaus well-short of the contextual DAG due to its inability to learn the varying ordering. The clustered DAG, which fits discrete DAGs to different clusters of $z$, likewise underperforms our continuous approach. Unsurprisingly, the fixed DAG fails to improve as $n$ increases.

Figure~\ref{fig:erdos-renyi-p} plots the results for Erdős-Rényi graphs as a function of $p$. All approaches naturally witness worsening performance as $p$ grows. In any case, the contextual DAG remains highly competitive with both sorted DAGs, the clustered DAG, and the fixed DAG.

Appendix~\ref{app:scale-free} reports the experiments for the scale-free graphs. Additional experiments for a setting where the contextual features are irrelevant (i.e., the fixed DAG's home court) are available in Appendix~\ref{app:fixed}.

\section{DRUG CONSUMPTION DATASET}
\label{sec:drug}

\begin{figure*}[t]
\centering
\begin{subfigure}[t]{0.3\linewidth}
\centering
\includegraphics[width=\linewidth]{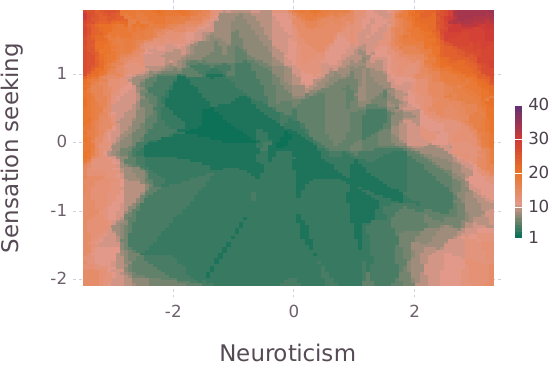}
\caption{Number of edges}
\label{fig:drug-sparsity}
\end{subfigure}
\begin{subfigure}[t]{0.3\linewidth}
\centering
\resizebox{0.8\linewidth}{!}{\includegraphics{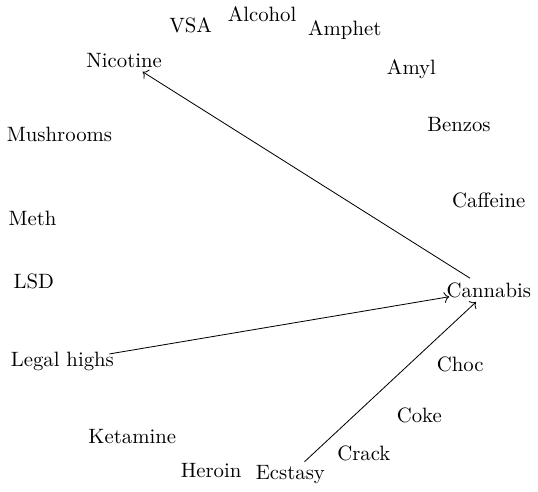}}
\caption{Low-scoring graph}
\label{fig:drug-g1}
\end{subfigure}
\begin{subfigure}[t]{0.3\linewidth}
\centering
\resizebox{0.8\linewidth}{!}{\includegraphics{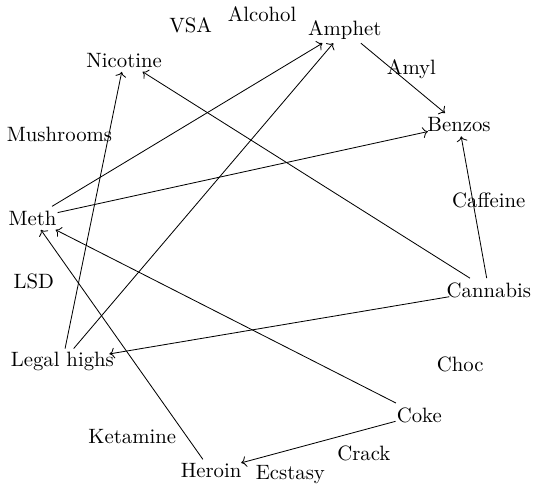}}
\caption{High-scoring graph}
\label{fig:drug-g2}
\end{subfigure}
\caption{The contextual DAG from the drug consumption dataset. The left plot is the graph sparsity as a function of the neuroticism/sensation seeking scores. The other plots are the graphs with the scores at low levels (0.1 quantiles) and high levels (0.9 quantiles). The parameter $\lambda$ is set to attain five edges on average over $z$.}
\end{figure*}

We consider a dataset from \textcite{Fehrman2017} on the recreational drug consumption patterns of $n=1885$ survey participants. The dataset includes measurements on the consumption of $p=18$ illicit and non-illicit drugs in terms of recency of use. In addition to these variables, the survey also captured the participants' personality characteristics. \textcite{Fehrman2017} reported neuroticism and sensation-seeking as the two most important determinants of drug consumption. Using these $m=2$ characteristics as contextual features, we apply the contextual DAG to predict personalized graphs of consumption dependencies.

Figure~\ref{fig:drug-sparsity} shows how the graph sparsity changes with the neuroticism and sensation seeking scores. Low and moderate scores yield sparser graphs than high scores, suggesting that individuals with atypical characteristics exhibit more complex and interconnected consumption patterns. Figures~\ref{fig:drug-g1} and \ref{fig:drug-g2} present the actual predicted graphs for two individuals: one with low scores and another with high scores. In the high-scoring graph, ``softer'' drugs like cannabis act as important nodes influencing the use of other substances. The edge between legal highs and cannabis is orientated differently in each graph, further indicating nuanced variation. Several edges in the high-scoring graph correspond to significant correlations identified in \textcite{Fehrman2017}, e.g., the edges between heroin and cocaine and methadone. The findings underscore the need for individualized risk mitigation strategies.

\section{SUMMARY}
\label{sec:summary}

Our paper introduces contextual DAGs, which relax the rigidity of fixed DAGs by allowing the graph structure to vary as a function of contextual features. A novel projection layer, for which we provide a convergence analysis and analytical gradients, allows us to learn neural networks that predict context-specific DAGs. An experimental analysis suggests that our approach can recover the true context-specific graph where other approaches fail. Our \texttt{Julia} implementation \texttt{ContextualDAG} is made publicly available.

\printbibliography

\section*{Checklist}

\begin{enumerate}

\item For all models and algorithms presented, check if you include:
\begin{enumerate}
\item A clear description of the mathematical setting, assumptions, algorithm, and/or model. [Yes]
\item An analysis of the properties and complexity (time, space, sample size) of any algorithm. [Yes]
\item (Optional) Anonymized source code, with specification of all dependencies, including external libraries. [Yes]
\end{enumerate}

\item For any theoretical claim, check if you include:
\begin{enumerate}
\item Statements of the full set of assumptions of all theoretical results. [Yes]
\item Complete proofs of all theoretical results. [Yes]
\item Clear explanations of any assumptions. [Yes]
\end{enumerate}

\item For all figures and tables that present empirical results, check if you include:
\begin{enumerate}
\item The code, data, and instructions needed to reproduce the main experimental results (either in the supplemental material or as a URL). [Yes]
\item All the training details (e.g., data splits, hyperparameters, how they were chosen). [Yes]
\item A clear definition of the specific measure or statistics and error bars (e.g., with respect to the random seed after running experiments multiple times). [Yes]
\item A description of the computing infrastructure used. (e.g., type of GPUs, internal cluster, or cloud provider). [Yes]
\end{enumerate}

\item If you are using existing assets (e.g., code, data, models) or curating/releasing new assets, check if you include:
\begin{enumerate}
\item Citations of the creator If your work uses existing assets. [Not Applicable]
\item The license information of the assets, if applicable. [Not Applicable]
\item New assets either in the supplemental material or as a URL, if applicable. [Not Applicable]
\item Information about consent from data providers/curators. [Not Applicable]
\item Discussion of sensible content if applicable, e.g., personally identifiable information or offensive content. [Not Applicable]
\end{enumerate}

\item If you used crowdsourcing or conducted research with human subjects, check if you include:
\begin{enumerate}
\item The full text of instructions given to participants and screenshots. [Not Applicable]
\item Descriptions of potential participant risks, with links to Institutional Review Board (IRB) approvals if applicable. [Not Applicable]
\item The estimated hourly wage paid to participants and the total amount spent on participant compensation. [Not Applicable]
\end{enumerate}

\end{enumerate}

\onecolumn

\begin{appendices}

\section{PROOF OF PROPOSITION~\ref{prop:proj}}
\label{app:proj}

\begin{proof}
For the proposition to hold, (1) the $\ell_1$ projection must not introduce any cycles and (2) $W^\star$ must be an element of the set $\mathbb{W}_s$. The first requirement is satisfied because the $\ell_1$ projection only modifies non-zero elements and therefore cannot add cycles. To show the second requirement is satisfied, we require Lemma~\ref{lemma:spectral}.

\begin{lemma}
\label{lemma:spectral}
Let $A=(\mathsf{a}_{jk})\in\mathbb{R}^{p\times p}$ and $B=(\mathsf{b}_{jk})\in\mathbb{R}^{p\times p}$ with $0\leq \mathsf{b}_{jk}\leq\mathsf{a}_{jk}$. Then it holds $\rho(B)\leq\rho(A)$, where $\rho(\cdot)$ is the spectral radius of its argument.
\end{lemma}
\begin{proof}
The Collatz-Wielandt formula for a non-negative matrix $A$ \parencite[see, e.g.,][\S 8.3]{Meyer2000} gives that
\begin{equation}
\label{eq:cw}
\rho(A)=\max_{x\in\mathbb{R}^p:x\geq0,x\neq0}\min_{j:\mathsf{x}_j\neq0}\frac{(Ax)_j}{\mathsf{x}_j}.
\end{equation}
Now, since $B$ is bounded elementwise by $A$, we have $(Bx)_j\leq(Ax)_j$ for any non-negative vector $x\neq0$, and hence
\begin{equation*}
\frac{(Bx)_j}{\mathsf{x}_j}\leq\frac{(Ax)_j}{\mathsf{x}_j},
\end{equation*}
for all $\mathsf{x}_j\neq0$. It follows immediately from the above inequality that
\begin{equation*}
\min_{j:\mathsf{x}_j\neq0}\frac{(Bx)_j}{\mathsf{x}_j}\leq\min_{j:\mathsf{x}_j\neq0}\frac{(Ax)_j}{\mathsf{x}_j}.
\end{equation*}
Taking the maximum over $x$ on both sides gives
\begin{equation*}
\max_{x\in\mathbb{R}^p:x>0,x\neq0}\min_{j:\mathsf{x}_j\neq0}\frac{(Bx)_j}{\mathsf{x}_j}\leq\max_{x\in\mathbb{R}^p:x>0,x\neq0}\min_{j:\mathsf{x}_j\neq0}\frac{(Ax)_j}{\mathsf{x}_j}.
\end{equation*}
By the Collatz-Wielandt formula \eqref{eq:cw}, the quantity on the left-hand side is $\rho(B)$ and the quantity on the right-hand side is $\rho(A)$. Hence, $\rho(B)\leq\rho(A)$.
\end{proof}

With Lemma~\ref{lemma:spectral} in hand, we take $A=\hat{W}\circ\hat{W}$ and $B=W^\star\circ W^\star$ to get $\rho(W^\star\circ W^\star)\leq\rho(\hat{W}\circ\hat{W})$. Moreover, since $\rho(\hat{W}\circ\hat{W})<s$, it holds $\rho(W^\star\circ W^\star)<s$ and hence $W^\star\in\mathbb{W}_s$.
\end{proof}

\section{PROOF OF THEOREM~\ref{thm:converge}}
\label{app:converge}

\subsection{Preliminary Lemmas}

The proof requires several technical lemmas which we state and prove in turn. The first lemma shows that $\nabla f_{\mu,s}(W;\tilde{W})$ is Lipschitz over a compact set of $W$. The second lemma provides that the gradient descent sequence $\{W^{(k)}\}_{k\in\mathbb{N}}$ stays in a compact set when the step size is suitably small. The third lemma ensures descent under Lipschitz continuity.

To facilitate exposition, we write $f_{\mu,s}(W,\tilde{W})$ as $f(W):=l(W)+h(W)$, where $l(W):=\mu/2\|W-\tilde{W}\|_F^2$ and $h(W)$ is defined as in \eqref{eq:hs}. The gradient $\nabla f(W)=\nabla l(W)+\nabla h(W)$, where $\nabla l(W)=-\mu/2(\tilde{W}-W)$ and $\nabla h(W)=2(sI-W\circ W)^{-\top}\circ W$. Except for the Frobenius norm, all norms below are induced matrix norms.

\begin{lemma}
\label{lemma:lipschitz}
Let $\tilde{W}\in\mathbb{R}^{p\times p}$ with $|\tilde{\mathsf{w}}_{jk}|\leq1$ and $\mu\geq0$. Define $\mathcal{W}:=\{W\in\mathbb{R}^{p\times p}:|\mathsf{w}_{jk}|\leq|\tilde{\mathsf{w}}_{jk}|\}$. Then, for any $s\geq1+\max(\|\tilde{W}\|_1,\|\tilde{W}\|_\infty)$, the gradient $\nabla f(W):\mathcal{W}\to\mathbb{R}^{p\times p}$ is Lipschitz continuous with respect to the Frobenius norm and has Lipschitz constant $c=\max(\mu/2,2\sqrt{p}+4p\|\tilde{W}\|_F)$.
\end{lemma}
\begin{proof}
Recall that a function $g(W):\mathcal{W}\to\mathbb{R}^{p\times p}$ is Lipschitz with respect to the Frobenius norm if, for all $W_1$ and $W_2$ in $\mathcal{W}$, it holds
\begin{equation}
\label{eq:lipschitz}
\|g(W_1)-g(W_2)\|_F\leq c\|W_1-W_2\|_F,
\end{equation}
for some $c\geq0$.

Now, if $\nabla l(W)$ and $\nabla h(W)$ are Lipschitz with constants $c_1$ and $c_2$, then $\nabla f(W)$ is also Lipschitz with constant $c=\max(c_1,c_2)$. Fix any $W_1\in\mathcal{W}$ and any $W_2\in\mathcal{W}$.

It is trivial to show that $\nabla l(W)$ satisfies the Lipschitz inequality \eqref{eq:lipschitz}, since
\begin{equation*}
\begin{split}
\|\nabla l(W_1)-\nabla l(W_2)\|_F&=\left\|-\frac{\mu}{2}(\tilde{W}-W_1)+\frac{\mu}{2}(\tilde{W}-W_2)\right\|_F \\
&=\frac{\mu}{2}\|W_1-W_2\|_F.
\end{split}
\end{equation*}
Hence, $\nabla l(W)$ is Lipschitz with constant $c_1=\mu/2$.

Next, for $\nabla h(W)$, we have
\begin{equation*}
\begin{split}
\|\nabla h(W_1)-\nabla h(W_2)\|_F&=\|2(sI-W_1\circ W_1)^{-\top}\circ W_1-2(sI-W_2\circ W_2)^{-\top}\circ W_2\|_F \\
&=2\|(sI-W_1\circ W_1)^{-\top}\circ W_1-(sI-W_2\circ W_2)^{-\top}\circ W_2\|_F.
\end{split}
\end{equation*}
To simplify the notation, we let $A_i=(sI-W_i\circ W_i)^\top$ for $i=1,2$. Observe now that
\begin{equation}
\label{eq:h1}
\begin{split}
2\|A_1^{-1}\circ W_1-A_2^{-1}\circ W_2\|_F&=2\|A_1^{-1}\circ W_1-A_1^{-1}\circ W_2+A_1^{-1}\circ W_2-A_2^{-1}\circ W_2\|_F \\
&\leq2\|A_1^{-1}\circ W_1-A_1^{-1}\circ W_2\|_F+2\|A_1^{-1}\circ W_2-A_2^{-1}\circ W_2\|_F \\
&\leq2\|A_1^{-1}\|_F\|W_1-W_2\|_F+2\|W_2\|_F\|A_1^{-1}-A_2^{-1}\|_F.
\end{split}
\end{equation}
The first and second inequalities follow from the Frobenius norm being sub-additive and sub-multiplicative, respectively.

Consider the first term on the far right-hand side of \eqref{eq:h1}. The requirements that $s\geq1+\max(\|\tilde{W}\|_1,\|\tilde{W}\|_\infty)$ and $W_1\in\mathcal{W}$ mean that $A_1$ is row-wise and column-wise diagonally dominant with a dominance factor of at least one. Then Corollary 2 of \textcite{Varah1975} applies and gives $\|A_1^{-1}\|_2\leq1$. This result in combination with the inequality $\|A_1^{-1}\|_F\leq\sqrt{p}\|A_1^{-1}\|_2$ yields
\begin{equation}
\label{eq:h2}
\|A_1^{-1}\|_F\|W_1-W_2\|_F\leq\sqrt{p}\|W_1-W_2\|_F.
\end{equation}

Consider now the second term on the far right-hand side of \eqref{eq:h1}. It holds
\begin{equation*}
\begin{split}
\|W_2\|_F\|A_1^{-1}-A_2^{-1}\|_F&=\|W_2\|_F\|A_1^{-1}(A_1-A_2)A_2^{-1}\|_F \\
&\leq\|W_2\|_F\|A_1^{-1}\|_F\|A_2^{-1}\|_F\|A_1-A_2\|_F \\
&\leq p\|W_2\|_F\|A_1-A_2\|_F \\
&=p\|W_2\|_F\|W_1\circ W_1-W_2\circ W_2\|_F \\
&\leq 2p\|W_2\|_F\|W_1-W_2\|_F.
\end{split}
\end{equation*}
The last inequality follows from the fact that $W_1$ and $W_2$ have elements bounded in absolute value by one, since
\begin{equation*}
\begin{split}
\|W_1\circ W_1-W_2\circ W_2\|_F&=\sqrt{\sum_{j=1}^p\sum_{k=1}^p(\mathsf{w}_{1jk}^2-\mathsf{w}_{2jk}^2)^2} \\
&=\sqrt{\sum_{j=1}^p\sum_{k=1}^p(\mathsf{w}_{1jk}+\mathsf{w}_{2jk})^2(\mathsf{w}_{1jk}-\mathsf{w}_{2jk})^2} \\
&\leq\sqrt{\sum_{j=1}^p\sum_{k=1}^p4(\mathsf{w}_{1jk}-\mathsf{w}_{2jk})^2} \\
&=2\|W_1-W_2\|_F.
\end{split}
\end{equation*}
Now, using that $\|W\|_F\leq\|\tilde{W}\|_F$ for any $W\in\mathcal{W}$, it follows
\begin{equation}
\label{eq:h3}
\|W_2\|_F\|A_1^{-1}-A_2^{-1}\|_F\leq2p\|\tilde{W}\|_F\|W_1-W_2\|_F.
\end{equation}

Plugging the bounds \eqref{eq:h2} and \eqref{eq:h3} into \eqref{eq:h1}, we arrive at
\begin{equation*}
\|\nabla h(W_1)-\nabla h(W_2)\|_F\leq(2\sqrt{p}+4p\|\tilde{W}\|_F)\|W_1-W_2\|_F.
\end{equation*}
Hence, $\nabla h(W)$ is Lipschitz with constant $c_2=2\sqrt{p}+4p\|\tilde{W}\|_F$.

The claim of the lemma now follows since $\nabla f(W)$ must also be Lipschitz with constant $c=\max(c_1,c_2)=\max(\mu/2,2\sqrt{p}+4p\|\tilde{W}\|_F)$.
\end{proof}

\begin{lemma}
\label{lemma:bounded}
Let $\tilde{W}\in\mathbb{R}^{p\times p}$ with $|\tilde{\mathsf{w}}_{jk}|\leq1$, $\mu\geq0$, and $c\geq\max(\mu/2,1)$. Define $\mathcal{W}:=\{W\in\mathbb{R}^{p\times p}:|\mathsf{w}_{jk}|\leq|\tilde{\mathsf{w}}_{jk}|\}$. Then, for all $W^{(k)}\in\mathcal{W}$ and any $s\geq1+\max(\|\tilde{W}\|_1,\|\tilde{W}\|_\infty)$, it holds $W^{(k+1)}\in\mathcal{W}$, where
\begin{equation*}
W^{(k+1)}=W^{(k)}-\frac{1}{c}\nabla f(W^{(k)}).
\end{equation*}
\end{lemma}
\begin{proof}
Assume without loss of generality that all elements of $\tilde{W}$ are non-negative. We aim to show that for any $W^{(k)}\in\mathcal{W}=\{W\in\mathbb{R}^{p\times p}:0\leq \mathsf{w}_{jk}\leq\tilde{\mathsf{w}}_{jk}\}$, the gradient descent update $W^{(k+1)}$ is also in $\mathcal{W}$. We break the proof into two parts: (1) we prove that $\mathsf{w}_{jk}^{(k+1)}\leq\tilde{\mathsf{w}}_{jk}$ and (2) we prove that $\mathsf{w}_{jk}^{(k+1)}\geq0$.

Fix any $\mathcal{W}^{(k)}\in\mathcal{W}$. For the upper bound (1), we have
\begin{equation}
\label{eq:upper}
\begin{split}
W^{(k+1)}&=W^{(k)}-\frac{1}{c}\nabla f(W^{(k)}) \\
&=W^{(k)}-\frac{1}{c}\left(\nabla l(W^{(k)})+\nabla h(W^{(k)})\right) \\
&\leq W^{(k)}-\frac{1}{c}\nabla l(W^{(k)}) \\
&=W^{(k)}+\frac{\mu}{2c}(\tilde{W}-W^{(k)}) \\
&\leq W^{(k)}+\tilde{W}-W^{(k)} \\
&=\tilde{W}.
\end{split}
\end{equation}
The first inequality follows from $\nabla h(W^{(k)})\geq0$, since $sI-W^{(k)}\circ W^{(k)}$ is a non-singular $M$ matrix and hence its inverse is non-negative. The last inequality follows from $\mu/(2c)\leq1$ and the elements of $\tilde{W}-W^{(k)}$ being non-negative, since $W^{(k)}\in\mathcal{W}$.

For the lower bound (2), it holds
\begin{equation*}
\begin{split}
W^{(k+1)}&=W^{(k)}-\frac{1}{c}\nabla f(W^{(k)}) \\
&=W^{(k)}-\frac{1}{c}\left(\nabla l(W^{(k)})+\nabla h(W^{(k)})\right) \\
&=W^{(k)}-\frac{1}{c}\left(\frac{\mu}{2}(W^{(k)}-\tilde{W})+\nabla h(W^{(k)})\right) \\
&=W^{(k)}+\frac{\mu}{2c}(\tilde{W}-W^{(k)})-\frac{1}{c}\nabla h(W^{(k)}) \\
&\geq W^{(k)}-\frac{1}{c}\nabla h(W^{(k)}) \\
&\geq W^{(k)}-\nabla h(W^{(k)}) \\
&=W^{(k)}-(sI-W^{(k)}\circ W^{(k)})^{-\top}\circ W^{(k)}.
\end{split}
\end{equation*}
The first inequality follows again from $\tilde{W}-W^{(k)}\geq0$. The second inequality follows from $1/c\leq 1$ and $\nabla h(W^{(k)})$ having non-negative elements. Now, observe that for a matrix $B=(\mathsf{b}_{jk})\in\mathbb{R}^{p\times p}$, it holds
\begin{equation*}
|\mathsf{b}_{jk}|=|e^{(j)\top}Be^{(k)}|\leq\|e^{(j)}\|_2\|Be^{(k)}\|_2=\|Be^{(k)}\|_2\leq\max_{x\in\mathbb{R}^p:\|x\|_2=1}\|Bx\|_2=\|B\|_2,
\end{equation*}
where $e^{(j)}$ is a standard basis vector (one in the $j$th position and zero elsewhere). The last equality follows from the definition of the spectral norm. Letting $B=(sI-W^{(k)}\circ W^{(k)})^{-\top}$, we have $|\mathsf{b}_{jk}|\leq\|B\|_2\leq1$ along the same lines as the argument for \eqref{eq:h2} in Lemma~\ref{lemma:lipschitz}. Combining this result with the previous inequality on $W^{(k+1)}$ gives
\begin{equation}
\label{eq:lower}
\begin{split}
W^{(k+1)}&\geq W^{(k)}-(sI-W^{(k)}\circ W^{(k)})^{-\top}\circ W^{(k)} \\
&\geq0.
\end{split}
\end{equation}

Combining the inequalities \eqref{eq:upper} and \eqref{eq:lower} gives the desired bound $0\leq W^{(k+1)}\leq\tilde{W}$.
\end{proof}

\begin{lemma}
\label{lemma:descent}
Let $g(W):\mathbb{R}^{p\times p}\to\mathbb{R}$. Suppose $\nabla g(W)$ is Lipschitz in Frobenius norm with Lipschitz constant $c>0$. Then, for any $W\in\mathbb{R}^{p\times p}$ and any $V\in\mathbb{R}^{p\times p}$, it holds
\begin{equation*}
g(V)\leq g(W)+\langle\nabla g(W),V-W\rangle+\frac{c}{2}\|V-W\|_F^2,
\end{equation*}
where $\langle\cdot,\cdot\rangle$ is the Frobenius inner product.
\end{lemma}
\begin{proof}
See the proof of Lemma~10 in \textcite{Minsker2022}.
\end{proof}

\subsection{Proof of Main Result}

We now prove Theorem~\ref{thm:converge}.

\begin{proof}
Take $\bar{c}\geq c$ as an upper bound to the Lipschitz constant $c$. The objective $f(W)$ can be lower bounded as
\begin{equation*}
\begin{split}
f(W)&=f(W)+\langle\nabla f(W),W-W\rangle+\frac{\bar{c}}{2}\|W-W\|_F^2 \\
&\geq\inf_{V\in\mathbb{R}^{p\times p}}\left(f(W)+\langle\nabla f(W),V-W\rangle+\frac{\bar{c}}{2}\|V-W\|_F^2\right) \\
&=\inf_{V\in\mathbb{R}^{p\times p}}\left(f(W)-\frac{1}{2\bar{c}}\|\nabla f(W)\|_F^2+\frac{\bar{c}}{2}\left\|V-\left(W-\frac{1}{\bar{c}}\nabla f(W)\right)\right\|_F^2\right).
\end{split}
\end{equation*}
Observe that the infimum is attained at the gradient descent update
\begin{equation*}
\hat{W}=W-\frac{1}{\bar{c}}\nabla f(W).
\end{equation*}
Substituting $\hat{W}$ into the previous inequality gives
\begin{equation}
\label{eq:descent1}
\begin{split}
f(W)&\geq f(W)-\frac{1}{2\bar{c}}\|\nabla f(W)\|_F^2+\frac{\bar{c}}{2}\left\|\hat{W}-\left(W-\frac{1}{\bar{c}}\nabla f(W)\right)\right\|_F^2 \\
&=f(W)+\langle\nabla f(W),\hat{W}-W\rangle+\frac{\bar{c}}{2}\|\hat{W}-W\|_F^2 \\
&=f(W)+\langle\nabla f(W),\hat{W}-W\rangle+\frac{c}{2}\|\hat{W}-W\|_F^2+\frac{\bar{c}-c}{2}\|\hat{W}-W\|_F^2.
\end{split}
\end{equation}
Lemmas \ref{lemma:lipschitz} and \ref{lemma:bounded} give that $\nabla f(W)$ is Lipschitz under the conditions of the theorem. This result is sufficient to invoke Lemma~\ref{lemma:descent} and lower bound the first three terms on the right-hand side as
\begin{equation}
\label{eq:descent2}
f(W)+\langle\nabla f(W),\hat{W}-W\rangle+\frac{c}{2}\|\hat{W}-W\|_F^2\geq f(\hat{W}).
\end{equation}
Substituting \eqref{eq:descent2} into \eqref{eq:descent1} yields
\begin{equation*}
f(W)\geq f(\hat{W})+\frac{\bar{c}-c}{2}\|\hat{W}-W\|_F^2.
\end{equation*}
Finally, taking $W=W^{(k)}$ and $\hat{W}=W^{(k+1)}$, we arrive at
\begin{equation}
\label{eq:converge}
f(W^{(k)})-f(W^{(k+1)})\geq\frac{\bar{c}-c}{2}\|W^{(k+1)}-W^{(k)}\|_F^2.
\end{equation}
Hence, the sequence $\{f(W^{(k)})\}_{k\in\mathbb{N}}$ is decreasing, and because $f(W)$ is bounded below by zero, it converges.
\end{proof}

\section{\texorpdfstring{$\ell_1$}{l1} PROJECTION ALGORITHM}
\label{app:l1}

Algorithm~\ref{alg:l1} provides the method for projecting onto the $\ell_1$ ball by computing the thresholding parameter $\kappa$.

\begin{algorithm}[H]
\small
\caption{$\ell_1$ projection}
\label{alg:l1}
\SetKwInOut{Input}{input}
\SetKwInOut{Output}{output}
\Input{Adjacency matrices $\hat{W}_1,\ldots,\hat{W}_n\in\mathbb{R}^{p\times p}$, regularization parameter $\lambda>0$}
Set $v\gets[\operatorname{vec}(\hat{W}_1),\ldots,\operatorname{vec}(\hat{W}_n)]$ \\
Take $u$ as the absolute values of $v$ \\
Sort $u$ in decreasing order as $\mathsf{u}_j>\mathsf{u}_k$ for all $j<k$ \\
Set $k_\mathrm{max}\gets\max\left\{k:\mathsf{u}_k>\left(\sum_{j=1}^k\mathsf{u}_j-n\lambda\right)/k\right\}$ \\
Set $\kappa\gets\left(\sum_{j=1}^{k_\mathrm{max}}\mathsf{u}_j-n\lambda\right)/k_\mathrm{max}$ \\
Compute $W_1^\star,\ldots,W_n^\star$ as $\mathsf{w}_{ijk}^\star\gets\operatorname{sgn}(\hat{\textsf{w}}_{ijk})\max(|\hat{\textsf{w}}_{ijk}|-\kappa,0)$ for $i=1,\ldots,n$ and $j,k=1,\ldots,p$ \\
\Output{Adjacency matrices $W_1^\star,\ldots,W_n^\star$}
\end{algorithm}

The algorithm directly extends that of \textcite{Duchi2008} for projecting a vector $\hat{v}\in\mathbb{R}^p$ onto the $\ell_1$ ball:
\begin{equation}
\label{eq:vectorproj}
\underset{v\in\mathbb{R}^p:\|v\|_1\leq\gamma}{\min}\;\frac{1}{2}\|\hat{v}-v\|_2^2.
\end{equation}
Our modified algorithm simply flattens the matrices into a vector and then reshapes the result of the projection. This approach is valid since \eqref{eq:vectorproj} is equivalent to 
\begin{equation*}
\underset{W_1,\ldots,W_n\in\mathbb{R}^{p\times p}:\frac{1}{n}\sum_{i=1}^n\|W_i\|_{\ell_1}\leq\lambda}{\min}\;\frac{1}{2}\sum_{i=1}^n\|\hat{W}_i-W\|_F^2
\end{equation*}
when $\hat{v}=[\operatorname{vec}(\hat{W}_1),\ldots,\operatorname{vec}(\hat{W}_n)]$, $v=[\operatorname{vec}(W_1),\ldots,\operatorname{vec}(W_n)]$, and $\gamma=n\lambda$.

\section{PROOF OF THEOREM~\ref{thm:gradient}}
\label{app:gradient}

\begin{proof}
To simplify exposition of the proof, we assume $n=1$ without loss of generality. We consider the cases where the $\ell_1$ constraint is binding and non-binding in turn.

\textbf{Non-Binding $\ell_1$ Constraint}

In the non-binding case where $\|W^\star\|_{\ell_1}<\lambda$, a solution $W^\star$ sets some elements of $\tilde{W}$ to zero via the DAG constraint and leaves the remaining elements untouched. There is no shrinkage from the $\ell_1$ constraint.

The $\mathsf{w}_{jk}^\star$ that are non-zero are an identity function of the input $\tilde{\mathsf{w}}_{qr}$ when $(q,r)=(j,k)$ and are a null function otherwise. More precisely, $\mathsf{w}_{jk}^\star(\tilde{\mathsf{w}}_{qr})=\tilde{\mathsf{w}}_{qr}$ for $(q,r)=(j,k)$ in the active set $\mathcal{A}$. The elements $\tilde{\mathsf{w}}_{qr}$ for which $(q,r)\neq(j,k)$ have no effect on $\mathsf{w}_{jk}^\star$, so $\mathsf{w}_{jk}^\star(\tilde{\mathsf{w}}_{qr})=0$ for $(q,r)\neq(j,k)$. It follows
\begin{equation}
\label{eq:activenonbinding}
\frac{\partial\mathsf{w}_{jk}^\star}{\partial\tilde{\mathsf{w}}_{qr}}=\delta_{jk}^{qr},
\end{equation}
for all $(j,k)\in\mathcal{A}$, where $\delta_{jk}^{qr}$ equals one if $(j,k)=(q,r)$ and zero otherwise.

The elements $(j,k)$ in the inactive set $\mathcal{A}^c$, where $\mathcal{A}^c$ is the complement of $\mathcal{A}$, are a null function in $\tilde{\mathsf{w}}_{qr}$, i.e., $\mathsf{w}_{jk}^\star(\tilde{\mathsf{w}}_{qr})=0$, even if $(q,r)=(j,k)$. It follows
\begin{equation}
\label{eq:inactivenonbinding}
\frac{\partial\mathsf{w}_{jk}^\star}{\partial\tilde{\mathsf{w}}_{qr}}=0,
\end{equation}
for all $(j,k)\in\mathcal{A}^c$.

Combining \eqref{eq:activenonbinding} and \eqref{eq:inactivenonbinding} yields the gradient for the non-binding case.

\textbf{Binding $\ell_1$ Constraint}

For the binding case where $\|W^\star\|_{\ell_1}=\lambda$, we derive the gradients by differentiating through the Karush-Kuhn-Tucker (KKT) conditions of \eqref{eq:project}, which characterize a solution in terms of $\tilde{W}$. We need not consider the DAG constraint in our analysis of these conditions since it only determines the active set and does not have any effect on the magnitude of the edge weights, unlike the $\ell_1$ constraint.

Let $\nu$ be the dual variable corresponding to the constraint $\|W\|_{\ell_1}\leq\lambda$. Like the primal solution $W^\star$, the dual solution $\nu^\star$ can be treated a function of $\tilde{W}$, i.e., $\nu^\star=\nu^\star(\tilde{\mathsf{w}}_{qr})$. To simplify the presentation, we omit the dependence hereafter and write $\nu^\star$ and the same for $W^\star$ (or $\mathsf{w}_{jk}^\star$).

Recall that the $\ell_1$-norm is not differentiable but subdifferentiable. We denote the subderivative of a function $f(x)$ as $\partial f(x)$. With this notation, the KKT conditions for stationarity and complementary slackness are
\begin{equation}
\label{eq:stationarity}
\partial\left(\frac{1}{2}\|\tilde{W}-W^\star\|_F^2+\nu^\star(\|W^\star\|_{\ell_1}-\lambda)\right)\ni0
\end{equation}
and
\begin{equation}
\label{eq:slackness}
\nu^\star(\|W^\star\|_{\ell_1}-\lambda)=0.
\end{equation}

We consider separately the gradients on the active set $\mathcal{A}$ and the gradients on the inactive set $\mathcal{A}^c$. We begin by deriving the gradients on $\mathcal{A}$. It follows immediately from the complementary slackness condition \eqref{eq:slackness} that
\begin{equation*}
\|W^\star\|_{\ell_1}-\lambda=\sum_{(j,k)\in\mathcal{A}}|\mathsf{w}_{jk}^\star|-\lambda=0,
\end{equation*}
and hence its gradient with respect to $\tilde{\mathsf{w}}_{qr}$ is
\begin{equation*}
\frac{\partial}{\partial\tilde{\mathsf{w}}_{qr}}\left(\sum_{(j,k)\in\mathcal{A}}|\mathsf{w}_{jk}^\star|-\lambda\right)=0.
\end{equation*}
Evaluating the derivative on the left-hand side yields
\begin{equation}
\label{eq:dnu1}
\sum_{(j,k)\in\mathcal{A}}\operatorname{sgn}(\mathsf{w}_{jk}^\star)\frac{\partial\mathsf{w}_{jk}^\star}{\partial\tilde{\mathsf{w}}_{qr}}=0.
\end{equation}
Now, evaluating the subderivative on the left-hand side of the stationarity condition \eqref{eq:stationarity} leads to the equalities
\begin{equation*}
\mathsf{w}_{jk}^\star-\tilde{\mathsf{w}}_{jk}+\nu^\star\operatorname{sgn}(\mathsf{w}_{jk}^\star)=0,
\end{equation*}
which hold for all $(j,k)\in\mathcal{A}$. Differentiating this expression with respect to $\tilde{\mathsf{w}}_{qr}$ gives
\begin{equation*}
\frac{\partial\mathsf{w}_{jk}^\star}{\partial\tilde{\mathsf{w}}_{qr}}-\delta_{jk}^{qr}+\frac{\partial\nu^\star}{\partial\tilde{\mathsf{w}}_{qr}}\operatorname{sgn}(\mathsf{w}_{jk}^\star)+\nu^\star\frac{\partial}{\partial\tilde{\mathsf{w}}_{qr}}\operatorname{sgn}(\mathsf{w}_{jk}^\star)=0,
\end{equation*}
which simplifies to
\begin{equation}
\label{eq:dbeta1}
\frac{\partial\mathsf{w}_{jk}^\star}{\partial\tilde{\mathsf{w}}_{qr}}=\delta_{jk}^{qr}-\frac{\partial\nu^\star}{\partial\tilde{\mathsf{w}}_{qr}}\operatorname{sgn}(\mathsf{w}_{jk}^\star).
\end{equation}
Substituting \eqref{eq:dbeta1} into \eqref{eq:dnu1} leads to
\begin{equation*}
\sum_{(j,k)\in\mathcal{A}}\operatorname{sgn}(\mathsf{w}_{jk}^\star)\left(\delta_{jk}^{qr}-\frac{\partial\nu^\star}{\partial\tilde{\mathsf{w}}_{qr}}\operatorname{sgn}(\mathsf{w}_{jk}^\star)\right)=0,
\end{equation*}
from which it follows
\begin{equation*}
\operatorname{sgn}(\mathsf{w}_{qr}^\star)-\frac{\partial\nu^\star}{\partial\tilde{\mathsf{w}}_{qr}}\sum_{(j,k)\in\mathcal{A}}\operatorname{sgn}(\mathsf{w}_{jk}^\star)^2=0.
\end{equation*}
Again, rearranging and simplifying, we have
\begin{equation}
\label{eq:dnu2}
\frac{\partial\nu^\star}{\partial\tilde{\mathsf{w}}_{qr}}=\frac{\operatorname{sgn}(\mathsf{w}_{qr}^\star)}{\sum_{(j,k)\in\mathcal{A}}\operatorname{sgn}(\mathsf{w}_{jk}^\star)^2} \\=\frac{1}{|\mathcal{A}|}\operatorname{sgn}(\mathsf{w}_{qr}^\star).
\end{equation}
Substituting \eqref{eq:dnu2} into \eqref{eq:dbeta1}, we obtain
\begin{equation}
\label{eq:activebinding}
\frac{\partial\mathsf{w}_{jk}^\star}{\partial\tilde{\mathsf{w}}_{qr}}=\delta_{jk}^{qr}-\frac{1}{|\mathcal{A}|}\operatorname{sgn}(\mathsf{w}_{qr}^\star)\operatorname{sgn}(\mathsf{w}_{jk}^\star),
\end{equation}
for all $(j,k)\in\mathcal{A}$. This quantity is the gradient of the projection's output $\mathsf{w}_{jk}^\star$ with respect to its input $\tilde{\mathsf{w}}_{qr}$ for the $\mathsf{w}_{jk}^\star$ that are non-zero.

Finally, for the gradients on $\mathcal{A}^c$, we see from \eqref{eq:slackness} that the complementary slackness condition does not involve the $\mathsf{w}_{jk}^\star$ such that $(j,k)\in\mathcal{A}^c$. The stationarity condition for $(j,k)\in\mathcal{A}^c$ is
\begin{equation*}
-\tilde{\mathsf{w}}_{jk}+\frac{\partial\nu^\star}{\partial\tilde{\mathsf{w}}_{qr}}s\ni0,
\end{equation*}
where $s=[-1,1]$ is the subderivative of the absolute value function at zero. Hence, the stationarity condition also does not involve $\mathsf{w}_{jk}^\star$. It follows that
\begin{equation}
\label{eq:inactivebinding}
\frac{\partial\mathsf{w}_{jk}^\star}{\partial\tilde{\mathsf{w}}_{qr}}=0,
\end{equation}
for all $(j,k)\in\mathcal{A}^c$. This quantity is the gradient of the projection's output $\mathsf{w}_{jk}^\star$ with respect to its input $\tilde{\mathsf{w}}_{qr}$ for the $\mathsf{w}_{jk}^\star$ that are zero.

Combining \eqref{eq:activebinding} and \eqref{eq:inactivebinding} yields the gradient for the binding case.
\end{proof}

\section{PATHWISE OPTIMIZATION}
\label{app:pathwise}

It is typical to compute a sequence of graphs corresponding to different sparsity levels from which the user can select. We take $\lambda$ as a sequence $\{\lambda^{(t)}\}_{t=1}^T$, where $\lambda^{(0)}$ imposes no regularization and $\lambda^{(T)}$ imposes full regularization. Except in degenerate cases, the unregularized model with $\lambda=\lambda^{(0)}$ contains $(p^2-p)/2$ edges, which is the maximum number of edges that an acyclic graph can have. The fully regularized model with $\lambda=\lambda^{(T)}$ contains no edges. Rather than compute these $T$ models independently, we compute them in a pathwise manner by sequentially warm-starting the optimizer. Specifically, the model for $\lambda^{(t+1)}$ is trained using the fitted weights $\hat{\theta}^{(t)}$ from the model for $\lambda^{(t)}$ as an initialization point.

Algorithm~\ref{alg:pathwise} presents the details of the pathwise optimization approach.

\begin{algorithm}[H]
\small
\caption{Pathwise optimization}
\label{alg:pathwise}
\SetKwInOut{Input}{input}
\SetKwInOut{Output}{output}
\Input{Initial network weights $\hat{\theta}^{(0)}\in\mathbb{R}^d$, step size $\varphi>0$, number of regularization parameters $T\in\mathbb{N}$}
Initialize $\lambda^{(1)}\gets\infty$ \\
\For{$t=1,\ldots,T$}{
Initialize $\theta_{(0)}\gets\hat{\theta}^{(t-1)}$ \\
Initialize $m\gets0$ \\
\While{Not converged}{
Update $\theta_{(m+1)}\gets\theta_{(m)}-\varphi\nabla_\theta L(\theta_{(m)};\lambda^{(t)})$ \\
Update $m\gets m+1$
}
Set $\hat{\theta}^{(t)}\gets\theta_{(m)}$ \\
\If{t=1}{
Set $\lambda^{(1)}\gets n^{-1}\sum_{i=1}^n\|W_{\hat{\theta}^{(1)}}(z_i)\|_{\ell_1}$ and $\lambda^{(T)}=0$ \\
Interpolate $T-2$ points between $\lambda^{(1)}$ and $\lambda^{(T)}$
}
}
\Output{Fitted network weights $\hat{\theta}^{(1)},\ldots,\hat{\theta}^{(T)}$}
\end{algorithm}

To unpack the notation, $L(\theta;\lambda)=n^{-1}\sum_{i=1}^n\|x_i-W_\theta(z_i)^\top x_i\|_2^2$ is the loss function evaluated at weights $\theta$ and regularization parameter $\lambda$. Its gradient is denoted $\nabla_\theta L(\theta;\lambda)$. The benefit of pathwise optimization is a significantly reduced overall runtime. Typically, it only takes a small number of iterations to converge if $\lambda^{(t+1)}$ and $\lambda^{(t)}$ are similar.

The initial network weights $\hat{\theta}^{(0)}$ for Algorithm~\ref{alg:pathwise} are computed by training a preliminary neural network that contains no projection layer. This preliminary network is initialized with random weights and is quick to train. Using the resulting trained weights as $\hat{\theta}^{(0)}$ yields better performance than directly taking $\hat{\theta}^{(0)}$ as random weights.

\section{IMPLEMENTATION OF METHODS}
\label{app:implementation}

All neural networks are implemented using two hidden layers, each containing 128 neurons with rectified linear activation functions. For the drug dataset, we use 32 neurons in each hidden layer due to the smaller sample size. Adam \parencite{Kingma2015} is used as an optimizer for learning the network weights, with a learning rate of 0.001. Convergence is monitored on a validation set with the optimizer terminated after 10 iterations without improvement. 

Following \textcite{Bello2022}, we apply a final thresholding step to all methods that sets near-zero edge weights to exact zeros, where the threshold is chosen to guarantee the graphs are acyclic. This step is necessary for continuous structure learning approaches such as NOTEARS and DAGMA because gradient descent and related first-order methods seldom produce machine-precision zeros.

For \emph{all methods}, we sweep the regularization parameter $\lambda$ over a grid of 20 values and take the final $\lambda$ that best approximates the correct sparsity level. This approach fairly compares graphs of the same complexity and ensures that the results are not confounded by uncertainty arising from model selection. In this regard, we follow earlier works such as \textcite{Zheng2018,Bello2022} that also forgo empirically estimating $\lambda$.

The experiments are performed on a Linux platform with two NVIDIA RTX 4090s.

\section{DESIGN OF SYNTHETIC DATASETS}
\label{app:synthetic}

We create synthetic datasets of $n$ observations as follows. First, we generate the contextual features $z_1,\ldots,z_n$ by taking $n$ iid draws uniformly on $[-1,1]^m$. The noise $\varepsilon_1,\ldots,\varepsilon_n$ is also generated by taking $n$ iid draws from a $p$-dimensional $N(0,I)$. The variables $x_1,\ldots,x_n$ are then generated as
\begin{equation*}
x_i=(I-W(z_i))^{-\top}\varepsilon_i,
\end{equation*}
which follows from rearranging terms in the model $x_i=W(z_i)^\top x_i+\varepsilon_i$. This process of sampling $z_1,\ldots,z_n$ and $x_1,\ldots,x_n$ is repeated three times to obtain independent training, validation, and testing sets.

The function $W(z)$ is constructed by randomly sampling an Erdős-Rényi or scale-free graph with 10 edges. This undirected graph is then oriented according to $z_i$ as follows. For each node $j$, we sample a point $c_j$ on $[-1,1]^m$. If the graph contains an edge between node $j$ and node $k$, we direct it from $j$ to $k$ if
\begin{equation*}
\|z_i-c_j\|_2-\|z_i-c_k\|_2>\phi,
\end{equation*}
and remove it otherwise. This scheme imposes a topological ordering where node $j$ is higher than node $k$ in the ordering if $\|z_i-c_j\|_2>\|z_i-c_k\|_2$. The parameter $\phi\geq0$ controls the expected sparsity of the graph (i.e., the average number of active edges). Larger values of $\phi$ encourage a sparser graph (in expectation), while smaller values have the opposite effect. We set $\phi$ empirically so that 5 of the 10 edges are active in the graph on average.

The edge weight $\mathsf{w}_{jk}$ is also allowed to vary with $z$ when non-zero. Specifically, for all edges between nodes $j$ and $k$, we set
\begin{equation*}
\mathsf{w}_{jk}(z)=
\begin{cases}
\|z_i-c_j\|_2-\|z_i-c_k\|_2 & \text{if }\|z_i-c_j\|_2-\|z_i-c_k\|_2>\phi \\
0 & \text{otherwise}.
\end{cases}
\end{equation*}
For all other pairs of nodes $j$ and $k$, we set $\mathsf{w}_{jk}(z)=0$.

\section{SCALE-FREE RESULTS}
\label{app:scale-free}

Figure~\ref{fig:scale-free} reports the scale-free graph results, analogous to the Erdős-Rényi graphs in Section~\ref{sec:synthetic}. The findings are quite similar to those for the Erdős-Rényi graphs---the contextual DAG is highly competitive at structure recovery. Its consistent performance across both settings indicates that it is reliable for multiple graph types.

\begin{figure*}[ht]
\centering
\begin{subfigure}[t]{\linewidth}
\centering
\includegraphics[width=0.8\linewidth]{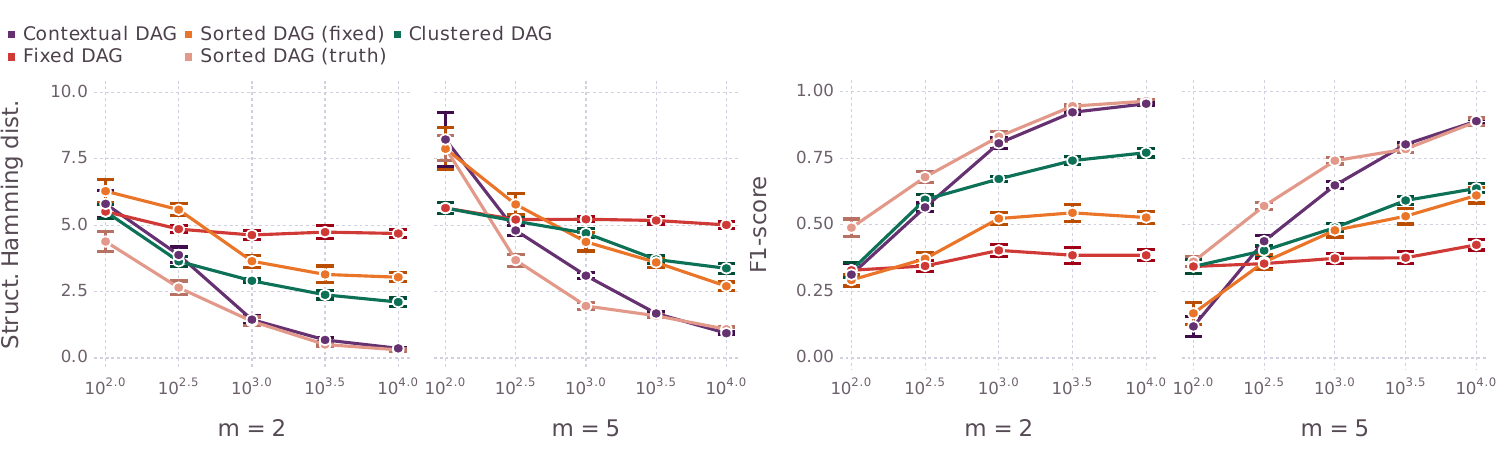}f
\caption{Structure recovery as a function of sample size $n$ for $m\in\{2,5\}$ contextual features.}
\end{subfigure}
\begin{subfigure}[t]{\linewidth}
\centering
\includegraphics[width=0.8\linewidth]{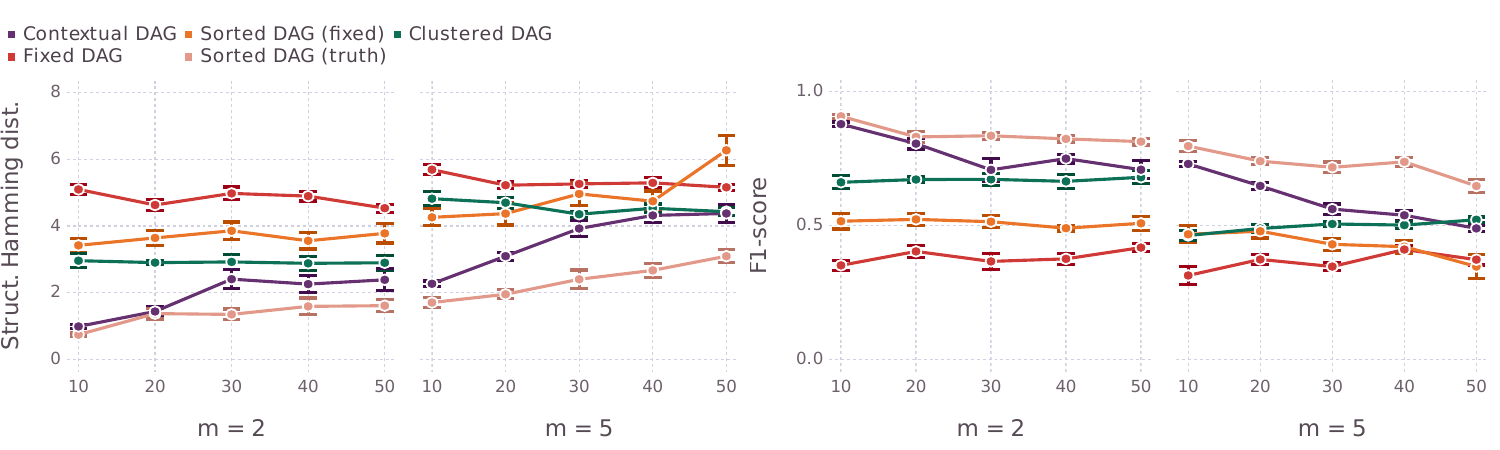}
\caption{Structure recovery as a function of number of nodes $p$ for $m\in\{2,5\}$ contextual features.}
\end{subfigure}
\caption{Structure recovery performance on varying scale-free graphs over 10 synthetic datasets. The number of nodes $p=20$ in the top row and the sample size $n=1000$ in the bottom row. The solid points are averages and the error bars are one standard errors. The sorted DAG (truth) uses the ground truth topological sort.}
\label{fig:scale-free}
\end{figure*}

\section{FIXED GRAPH RESULTS}
\label{app:fixed}

Figure~\ref{fig:fixed} reports the fixed graph results where the contextual features are irrelevant, i.e., $W(z)$ is constant in $z$. Although the inductive bias of methods that assume a single population DAG does have an impact, the contextual DAG remains competitive and still recovers the ground truth for large $n$ despite $z$ having no predictive value

\begin{figure*}[ht]
\centering
\includegraphics[width=0.8\linewidth]{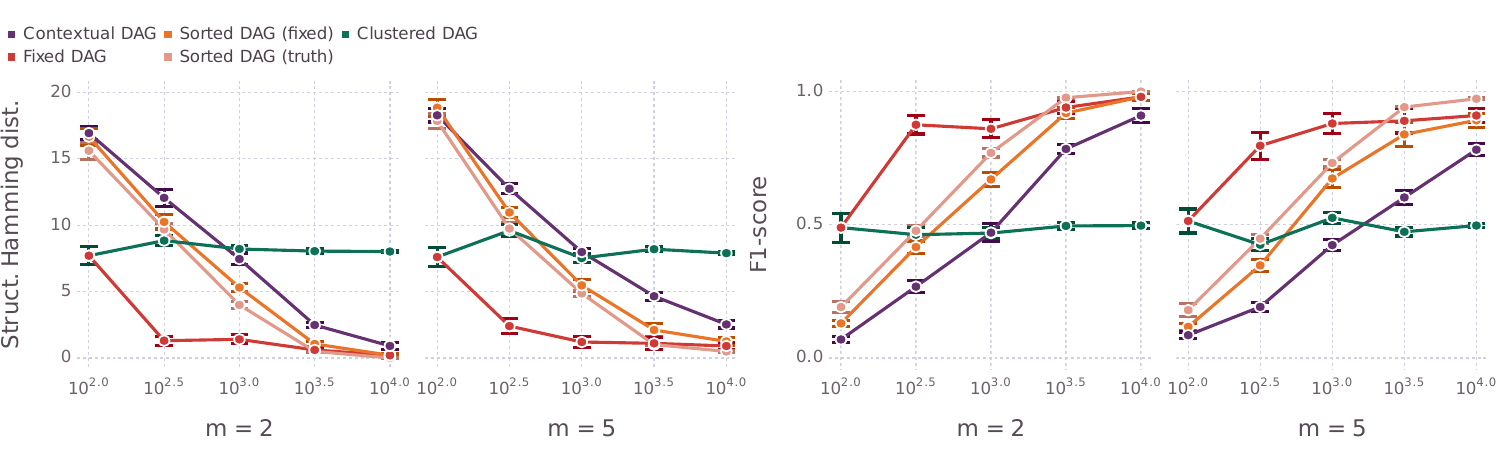}
\caption{Structure recovery performance on fixed Erdős-Rényi graphs over 10 synthetic datasets. The $x$-axis is sample size $n$. The number of contextual features $m\in\{2,5\}$ and the number of nodes $p=20$. The solid points are averages and the error bars are one standard errors. The sorted DAG (truth) uses the ground truth topological sort.}
\label{fig:fixed}
\end{figure*}

\end{appendices}

\end{document}